\documentclass{article}
\usepackage[preprint]{neurips_2023}
\usepackage[utf8]{inputenc}
\usepackage[algo2e]{algorithm2e}
\usepackage{times}
\usepackage{subfiles}
\usepackage{amssymb} 
\usepackage{amsmath} 
\usepackage{bbm}
\usepackage{tikz}
\usetikzlibrary{positioning,angles,quotes}
\usepackage{bm} 
\usepackage{multicol}
\usepackage{algorithm}
\usepackage{algpseudocode}
\usepackage[makeroom]{cancel}
\usepackage{xcolor}

\usepackage[hidelinks]{hyperref}
\usepackage{thm-restate}
\usepackage{mathtools} 
\usepackage{multicol}
\allowdisplaybreaks[3]
\mathtoolsset{showonlyrefs=true}  
\usepackage{amsthm}
\usepackage{thmtools}
\newtheorem{theorem}{Theorem}[section]
\newtheorem{lemma}[theorem]{Lemma}

\newtheorem{assumption}[theorem]{Assumption}
\newtheorem{proposition}[theorem]{Proposition}
\newtheorem{corollary}[theorem]{Corollary}

\newcommand{\repeattheorem}[1]{%
  \begingroup
  \renewcommand{\thetheorem}{\ref{#1}}%
  \expandafter\expandafter\expandafter\theorem
  \csname reptheorem@#1\endcsname
  \endtheorem
  \endgroup
}

\NewEnviron{reptheorem}[1]{%
  \global\expandafter\xdef\csname reptheorem@#1\endcsname{%
    \unexpanded\expandafter{\BODY}%
  }%
  \expandafter\theorem\BODY\unskip\label{#1}\endtheorem
}

\newcommand{\repeatlemma}[1]{%
  \begingroup
  \renewcommand{\thetheorem}{\ref{#1}}%
  \expandafter\expandafter\expandafter\lemma
  \csname reptheorem@#1\endcsname
  \endtheorem
  \endgroup
}

\NewEnviron{replemma}[1]{%
  \global\expandafter\xdef\csname reptheorem@#1\endcsname{%
    \unexpanded\expandafter{\BODY}%
  }%
  \expandafter\lemma\BODY\unskip\label{#1}\endtheorem
}

\newcommand{\repeatcorollary}[1]{%
  \begingroup
  \renewcommand{\thetheorem}{\ref{#1}}%
  \expandafter\expandafter\expandafter\corollary
  \csname reptheorem@#1\endcsname
  \endtheorem
  \endgroup
}

\NewEnviron{repcorollary}[1]{%
  \global\expandafter\xdef\csname reptheorem@#1\endcsname{%
    \unexpanded\expandafter{\BODY}%
  }%
  \expandafter\corollary\BODY\unskip\label{#1}\endtheorem
}

\newcommand{\repeatproposition}[1]{%
  \begingroup
  \renewcommand{\thetheorem}{\ref{#1}}%
  \expandafter\expandafter\expandafter\proposition
  \csname reptheorem@#1\endcsname
  \endtheorem
  \endgroup
}

\NewEnviron{reppro}[1]{%
  \global\expandafter\xdef\csname reptheorem@#1\endcsname{%
    \unexpanded\expandafter{\BODY}%
  }%
  \expandafter\proposition\BODY\unskip\label{#1}\endtheorem
}

\newcommand{\R}{\mathbb{R}}
\newcommand{\N}{\mathbb{N}}
\newcommand{\E}{\mathbb{E}}
\newcommand{\Hh}{\mathbb{H}}
\newcommand{\Pp}{\mathbb{P}}

\newcommand{\oneproj}{\Pcal_W}
\newcommand{\Scal}{\mathcal{S}}
\newcommand{\Acal}{\mathcal{A}}
\newcommand{\SAcal}{\Scal\times\Acal}

\newcommand{\Dcal}{\mathcal{D}}
\newcommand{\Hcal}{\mathcal{H}}
\newcommand{\Rcal}{\mathcal{R}}

\newcommand{\Ocal}{\mathcal{O}}
\newcommand{\Pcal}{\mathcal{P}}

\newcommand{\Xcal}{\mathcal{X}}
\newcommand{\Ycal}{\mathcal{Y}}

\newcommand{\abs}[1]{\left| {#1} \right|}
\newcommand{\bc}[1]{\left\{{#1}\right\}}
\newcommand{\br}[1]{\left({#1}\right)}
\newcommand{\bs}[1]{\left[{#1}\right]}
\newcommand\ip[2]{\langle #1, #2 \rangle}
\newcommand\kl[2]{\operatorname{KL}\br{ #1 || #2 }}
\newcommand{\norm}[1]{\left\lVert#1\right\rVert}

\title{Convergence of a model-free entropy-regularized \\inverse reinforcement learning algorithm}

\author{%
  Titouan Renard\thanks{The first three authors contributed equally. Andreas Schlaginhaufen, Tingting Ni, and Maryam Kamgarpour are with the SYCAMORE lab at École Polytechnique Fedérale de Lausanne (EPFL), Switzerland. Correspondence to:
 \{andreas.schlaginhaufen, tingting.ni\}@epfl.ch}
   \And
Andreas Schlaginhaufen\footnote[1]{}\hspace{0.15cm}\thanks{Andreas Schlaginhaufen was supported by the Swiss Data Science Center.} \\
   \And
  Tingting Ni\footnote[1]{}\hspace{0.15cm}\thanks{Tingting Ni was supported by the Swiss National Science Foundation.} \\
   \And
   Maryam Kamgarpour \\
}

\begin{document}
\maketitle
\begin{abstract}
    Given a dataset of expert demonstrations, inverse reinforcement learning (IRL) aims to recover a reward for which the expert is optimal. This work proposes a model-free algorithm to solve the entropy-regularized IRL problem. In particular, we employ a stochastic gradient descent update for the reward and a stochastic soft policy iteration update for the policy. Assuming access to a generative model, we prove that our algorithm is guaranteed to recover a reward for which the expert is $\varepsilon$-optimal using an expected
    number of $\mathcal{O}(1/\varepsilon^{2})$ samples of the Markov decision process (MDP). Furthermore, with an expected
    number of $\mathcal{O}(1/\varepsilon^{4})$ samples we prove that the optimal policy corresponding to the recovered reward is $\varepsilon$-close to the expert policy in total variation distance. 
\end{abstract}
\section{Introduction}
The problem of inverse reinforcement learning (IRL) can be informally characterized as follows. Given observations of an expert acting optimally with respect to an unknown reward in a Markov decision process (MDP), we aim to recover a reward for which the expert is optimal. While early inquiries can be traced back to optimal control theory \cite{Kalman1964} and to econometrics \cite{Rust1992}, IRL was first introduced to the machine learning community by Russell \cite{Russell1998}. The motivation behind IRL is two-fold: First, IRL is a powerful tool for imitation learning, where the goal is to recover the expert's policy from a dataset of expert demonstrations. This is particularly useful in scenarios such as autonomous driving \cite{pomerleau1988alvinn} where it is easy to collect expert demonstrations. Second, compared to other imitation learning methods that only recover a policy, IRL comes with the advantage that recovering a reward provides a more interpretable and transferable description of the task, as the reward can be potentially used to learn optimal policies in a new environment.

A powerful family of imitation learning and IRL algorithms are based on a min-max game between a policy and a reward player \cite{Syed2007}. The policy player tries to maximize the expected reward, while the reward player tries to minimize the suboptimality of the expert demonstrations. At a saddle point, we recover a reward for which the expert is nearly optimal along with a policy that approximates the expert's. Hence, this game-theoretic approach is useful for both imitation learning, by recovering the expert's policy, and for IRL, by recovering a reward that rationalizes the expert's behavior. However, when applied to IRL, the non-uniqueness of the optimal policy leads to trivial solutions, like a uniform reward for which all policies are optimal. This degeneracy is usually addressed through an entropy regularization term, which guarantees the uniqueness of the optimal policy, leading to the widely used entropy-regularized IRL framework \cite{ziebart2010modeling, Ho2016}.

Although approaches based on the aforementioned min-max game have shown tremendous success in imitation learning and IRL applications \cite{Ho2016, fu2017learning}, only little is known about their convergence properties. From an optimization perspective, the min-max problem introduced by \cite{Syed2007} is challenging for two reasons. First, since the policy player aims to maximize its expected reward in an MDP, the objective is inherently non-concave in the policy \cite{Agarwal2020}. Second, as in practice we are often only given sample-based access to the MDP, either via a simulator or by collecting samples in the real world, we need to resort to stochastic optimization techniques to build a \emph{model-free} algorithm. 

\subsection{Related work}
Most previous work on the min-max approach to imitation learning and IRL considers the unregularized setting and linear reward classes. For the infinite-horizon unregularized case, the authors of \cite{Syed2007} propose a multiplicative weights algorithm to update the reward and show that their algorithm provably recovers a policy that is approximately optimal under the expert's true reward. However, they require access to an approximate solution of the forward MDP problem at each update step of the reward. This results in an inner RL loop, which may lead to high per-iteration costs in practice. For infinite-horizon linear MDPs, the authors of \cite{viano2022proximal} propose a single-loop proximal point algorithm based on the so-called Q-LP formulation of the forward problem \cite{bas2021logistic}. They show that the policy learned by their algorithm efficiently convergences to the expert's, where the convergence is measured in a so-called integral probability metric \cite{muller1997integral} between the state-action occupancy measures.

For the finite-horizon unregularized case, the authors of \cite{Shani2021} analyze a single-loop mirror descent-ascent algorithm and show their algorithm achieves sublinear regret measured in terms of the worst-case difference between the accumulated values of the learner and the expert. Furthermore, the authors of \cite{Liu2022} extend this approach to linear MDPs with function approximation. It can be shown that the regret bounds established by \cite{Shani2021} and \cite{Liu2022} imply convergence of a randomized policy derived from the algorithm iterates in the aforementioned integral probability metric.

Note that all of the above works consider the unregularized setting and their guarantees are on the recovered policy rather than on the reward. To the best of our knowledge, \cite{Zeng2022} is the only work to provide guarantees for the infinite-horizon entropy-regularized setting. Their method uses an exact soft-policy iteration step to update the policy and stochastic gradient descent to update the reward. They show that the proposed algorithm reaches an $\varepsilon$-stationary solution in $\Ocal(1/\varepsilon^2)$ iterations. However, they do not provide guarantees on the suboptimality of the expert with respect to the recovered reward. Moreover, they require access to the exact state-action value function and an infinitely long trajectory to estimate the reward gradient, which makes the algorithm unimplementable in a model-free setting.


\subsection{Contributions}
We propose a model-free single-loop entropy-regularized IRL algorithm with a stochastic projected gradient descent update for the reward and a stochastic soft policy iteration  \cite{Haarnoja2018} update for the policy. Assuming access to a generative model, we show that our algorithm is guaranteed to recover a reward for which the expert is $\varepsilon$-optimal using an expected number of  $\Ocal(1/\varepsilon^2)$ samples from the MDP. Furthermore, we prove that with an expected number of $\Ocal(1/\varepsilon^4)$ samples, the optimal policy corresponding to the recovered reward is $\varepsilon$-close to the expert policy in total variation distance. To the best of our knowledge, this is the first work to provide end-to-end guarantees on the rewards for a model-free single-loop entropy-regularized IRL algorithm.

\section{Background}
\subsection{Notation}
We use $\R$ and $\N$ to denote the set of real and natural numbers, respectively. For a vector $x$ in $\R^d$, we denote its $p$-norm by $\norm{x}_p$ and its projection onto a closed convex set $\Xcal\subset \R^d$ by \(\mathcal{P}_{\mathcal{X}}(x) = \arg\min_{y\in\mathcal{X}} \|x-y\|_2\). For any two vectors $x,y\in\R^d$, we denote the standard inner product by $\langle x,y\rangle$. Furthermore, we let $\Ycal^{\Xcal}$ be the set of all functions mapping from the set $\Xcal$ to $\Ycal$. Given a finite set $\Xcal$ the infinity-norm of a vector valued function \(f:\Xcal\to\mathbb{R}^{d}\) over $\Xcal$ is defined as \(\|f\|_{\infty} :=\max_{x\in\Xcal}\norm{f(x)}_\infty\). Moreover, we denote the probability simplex over \(\Xcal\) by \(\Delta_{\Xcal}\). For $p\in\Delta_{\Xcal}$, we denote the Shannon entropy of $p$ by $\Hcal(p):= -\sum_{x\in\Xcal} p(x)\log p(x)$. Finally, a random variable $X$ taking values in $\N$ is geometrically distributed with parameter $p\in(0,1)$, denoted as $X\sim \operatorname{Geom}(p)$, if \(\Pr(X=k)=(1-p)^{k}p\).

\subsection{Markov decision processes}
Throughout this paper, we consider an entropy-regularized MDP defined by the tuple $(\mathcal{S}, \mathcal{A}, P, \nu_0, r, \gamma, \tau)$. Here, $\mathcal{S}$ and $\mathcal{A}$ denote finite state and action spaces, $P\in \Delta^{\SAcal}_{\Scal}$ is a Markovian transition kernel, and $\nu_0\in\Delta_{\Scal}$ is the initial state distribution. Moreover, $r\in\R^{\SAcal}$ is a reward function, $\gamma\in (0,1)$ a discount factor, and $\tau>0$ is the regularization parameter. Starting from $s_0\sim \nu_0$, the agent chooses at each step in time, $h$, an action $a_h\in\Acal$, receives reward $r(s_h, a_h)$, and arrives in state $s_{h+1}\sim P(\cdot|s_h, a_h)$. Given a stationary Markov policy $\pi\in\Delta_{\Acal}^{\Scal}$, the agent selects at each state $s_h$ its next action $a_h$ by sampling from $\pi(\cdot|s_h)$. We use $\Pp^{\pi}_{\nu_0}$ to denote the distribution over the sample space $(\Scal\times\Acal)^\infty = \bc{(s_0, a_0, s_1, a_1, \hdots): s_h\in\Scal, a_h\in\Acal,\; h\in\N}$ induced by the policy $\pi$ and the initial distribution $\nu_0$. Moreover, we let $\E_\pi$ be the expectation with respect to $\Pp^{\pi}_{\nu_0}$.

The goal of the \emph{forward} MDP problem is to find a policy $\pi\in\Delta_{\Acal}^{\Scal}$ maximizing the entropy-regularized expected discounted reward
\begin{align}\label{eq:mdp_objective}\tag{O-RL}
    J^\pi_r
        := \E_\pi \bs{\sum_{h=0}^{\infty} \gamma^h r(s_h,a_h)} + \tau \Hh(\pi),
\end{align}
where $\Hh(\pi):= \E_\pi \bs{\sum_{h=0}^\infty \gamma^h \Hcal\br{\pi(\cdot|s_h)}}$ is the  discounted causal entropy of $\pi$ \cite{ziebart2010modeling}. For a fixed reward $r$, we denote the optimal policy as
\begin{align}
    \pi^*_r := \arg \max_{\pi \in \Delta_{\Acal}^{\Scal}} J^\pi_r, \label{eq:mdp_optimality_formulation}
\end{align}
\noindent and the optimal objective value as $J_r^*:=J_r^{\pi^*_r}$. Note that the entropy regularization ensures that $\pi^*_r$ is unique \cite{geist2019theory}. To help characterize expectations over trajectories, we introduce the state occupancy measure $\nu^\pi\in\Delta_{\Scal}$ defined by
\begin{equation}
    \nu^\pi(s) := (1-\gamma) \sum_{h=0}^\infty \gamma^h \Pp_{\nu_0}^\pi(s_h = s).
\end{equation}
\noindent For a function $f\in\R^{\Scal}$, this allows us to rewrite the expectation $\E_\pi\bs{\sum_{h=0}^\infty \gamma^h f(s)}$ as $\E_{s\sim\nu^\pi}\bs{f(s)}/(1-\gamma)$ \cite{Puterman1994}.

\subsection{Problem statement}
The IRL problem is specified as follows: given access to a dataset of expert trajectories,
\begin{equation}
    \Dcal^E = \bc{\br{s_0^i,a_0^i, \hdots, s_{H-1}^i, a_{H-1}^i}}_{i=1}^{N},
\end{equation}
\noindent sampled from an unknown expert policy $\pi^E$, we aim to recover a reward $\bar{r}$, in some reward class $\Rcal\subseteq \R^{\SAcal}$, for which the expert is optimal. However, this is problematic for two reasons. First, the expert policy $\pi^E$ may not be optimal for any reward in $\Rcal$. Second, we only have access to $\Dcal^E$ rather than to $\pi^E$ itself. To address the first issue, we relax our goal to recovering a reward minimizing the suboptimality of the expert. This leads us to the min-max formulation
\begin{equation}\label{eq:idealized_irl}
    \min_{r\in\Rcal} \max_{\pi\in\Delta_{\Acal}^{\Scal}}J^{\pi}_r - J^{\pi^E}_r.
\end{equation}
We can interpret \eqref{eq:idealized_irl} as a zero-sum game between a policy player that aims to maximize the MDP objective \eqref{eq:mdp_objective} and a reward player that aims to minimize the suboptimality of the expert policy. If the expert is optimal for some reward in $\Rcal$, then the set of minimizers in \eqref{eq:idealized_irl} coincides with the set of rewards for which the expert is optimal \cite{Schlaginhaufen2023}.

To address the second issue, we need to estimate $J^{\pi^E}_r$ from the expert data set $\Dcal^E$. To this end, we consider the bounded linear reward class
\begin{equation}
    \Rcal = \bc{r_w := \ip{w}{\phi(\cdot,\cdot)} : \phi:\SAcal\to\R^k, w\in W},
\end{equation}
where $w$ is a weight vector in $W:= \bc{w\in\R^k:\norm{w}_1\leq 1}$ and $\phi(s,a)$ is a feature vector in $\R^k$. The features can either be task-specific, such as distance to target and speed in a driving task, or general, allowing for all state-only or state-action rewards, resulting in $k=|\Scal|$ or $k=|\Scal||\Acal|$, respectively. Introducing the feature expectation
\begin{equation}\label{featureexpectation}
    \sigma^\pi := \E_\pi\bs{\sum_{h=0}^\infty \gamma^h \phi(s_h, a_h)},
\end{equation}
associated with the policy $\pi$, and the empirical expert feature expectation 
\begin{equation}\label{expertfeature}
    \hat{\sigma}^E := \frac{1}{N}\sum_{i=1}^{N}\sum_{h=0}^{H-1} \gamma^h \phi(s_h^i, a_h^i),
\end{equation}
we can rewrite $J_{r_w}^\pi = \ip{w}{\sigma^\pi} + \tau\Hh(\pi)$ and estimate the expected reward of the expert policy as $\E_{\pi^E}\bs{\sum_{h=0}^\infty \gamma^t r_w(s_h, a_h)} \approx \ip{w}{\hat{\sigma}^E}$. Plugging this back into \eqref{eq:idealized_irl} leads us to the IRL problem
\begin{align}\label{eq:irl_problem}\tag{O-IRL}
    \min_{w\in W} \max_{\pi\in \Delta_{\Acal}^{\Scal}} \underbrace{\ip{w}{\sigma^\pi - \hat{\sigma}^E} + \tau \Hh(\pi)-\tau \Hh(\pi^E)}_{=:L(\pi, w)},\hspace{0.8cm}
\end{align}

\noindent where compared to \eqref{eq:idealized_irl} we replaced $\ip{w}{\sigma^{\pi^E}}$ with its empirical estimate $\ip{w}{\hat \sigma^E}$. The unknown term $\Hh(\pi^E)$ is constant in both parameters $w$ and $\pi$ and thus, is irrelevant for optimization. Note that for a fixed $w$, the maximizer of $L(\cdot,w)$ is the optimal policy $\pi_{r_w}^*$ defined in \eqref{eq:mdp_optimality_formulation}. Hence, the inner maximization in \eqref{eq:irl_problem} is equivalent to the forward MDP problem \eqref{eq:mdp_objective}. Our goal is to find $\bar{w}$ minimizing \eqref{eq:irl_problem}. An approximate expert policy can then be recovered from $\bar w$ by solving for the optimal policy corresponding to $r_{\bar w}$.
\section{Proposed algorithm}
To solve the problem described in \eqref{eq:irl_problem} with exact gradient information, we employ soft policy iteration for policy updates and projected gradient descent for reward updates. As detailed below, this can be viewed as a simple gradient descent-ascent type algorithm.
\paragraph{Policy update} For a known transition law $P$, we can maximize the objective in \eqref{eq:mdp_objective} via regularized dynamic programming \cite{geist2019theory}. In particular, for a fixed reward, we update the policy as
\begin{equation}\label{eq:updatepolicy}
    \pi_r^{t+1}(\cdot|s)\propto \exp\br{Q_r^{\pi^t}(s,\cdot)/\tau},
\end{equation}
with the state-action value function 
\begin{align}
   & Q_r^{\pi^t}(s,a) := r(s, a) +\E_{\pi}\bs{\sum_{h=1}^\infty \gamma^h \br{r(s_h, a_h) +\tau \Hcal(\pi(\cdot|s_h))}|s_0=s,a_0=a}.\label{Q}
\end{align}   
\noindent The above policy update is known as soft policy iteration, which converges linearly to the optimal policy \cite{geist2019theory}. Additionally, it can be shown that the soft policy iteration update coincides with entropy-regularized natural policy gradient with a specific stepsize \cite{Cen2022}. Hence, the policy update \eqref{eq:updatepolicy} can be interpreted as a gradient ascent update. However, note that computing the update \eqref{eq:updatepolicy} requires access to the state-action value function $Q_r^{\pi^t}(s,a)$ for all state-action pairs $(s,a)$.

\paragraph{Reward update} To find the optimal reward that minimizes the objective in \eqref{eq:irl_problem}, we update the reward parameter $w$ using a projected gradient descent update
 \begin{align}\label{eq:updatereward}
    w^{t+1}&=\oneproj\br{w^t - \eta_w \frac{\partial L(\pi^t,w)}{\partial w}\Big|_{w=w^t} }=\oneproj \br{ w^{t} - \eta_w \br{\sigma^{\pi^t}-\hat \sigma^{E} }},  
\end{align}
where $\oneproj$ represents the orthogonal projection onto $W$, and $\eta_w$ is the reward learning rate.

\subsection{Sampling scheme}
Both the policy update \eqref{eq:updatepolicy} and the reward update \eqref{eq:updatereward} require access to the transition law $P$ for evaluating the state-action value function \eqref{Q} and the feature expectation~\eqref{featureexpectation}, respectively. Hence, to devise a model-free algorithm, we need to estimate the state-action value and the feature expectation from samples of the MDP. To this end, we adopt the geometric sampling scheme, described below, which enables us to get unbiased estimates of the state-action value and the feature expectation.

\paragraph{State-action value estimation} We assume access to a generative model of the MDP, allowing us to obtain multiple independent trajectories starting from any arbitrary state-action pair with any policy. Given a policy $\pi$, for each state-action pair $(s,a)$, we construct an unbiased estimate of the state-action value $Q_{r}^{\pi}(s,a)$ by sampling $B$ independent trajectories $\tau_i := (s, a, \{s_h^i, a_h^i\}_{h=0}^{H_i})_{i=1}^{B}$ from $\pi$, with horizon $H_i\sim  \operatorname{Geom}(1-\gamma)$. As shown in \cite[Assumption 6.3]{Agarwal2020}, the estimator
\begin{align*}
    &\hat{Q}^{\pi}_r(s,a):=r(s,a)+\frac{1}{B}\sum_{i=1}^{B}\sum_{h=0}^{H_i}\br{r(s_h^i,a_h^i)+\tau \Hcal(\pi(\cdot|s_h^i))}.
\end{align*}
is an unbiased estimator of ${Q}^{\pi}_r(s,a)$. In the following, we will denote the above sampling procedure outputting $\hat{Q}^{\pi}_r(s,a)$ as $\operatorname{Est}Q(s,a,\pi,r, B)$.

\paragraph{Feature expectation estimation} Similarly, as for the state-action value estimate, we construct an unbiased estimate of the feature expectation $\sigma^\pi$, by sampling $B$ independent trajectories $\tau_i := (\{s_h^i, a_h^i\}_{h=0}^{H_i})_{i=1}^{B}$ from $\pi$, with horizon $H_i\sim  \operatorname{Geom}(1-\gamma)$ and initial state $s_0^i\sim\nu_0$. As shown in Lemma \ref{Unbiased_A_sampler}, the estimator
\begin{align*}
   \hat\sigma^\pi:=\dfrac{1}{B}\sum_{i=1}^{B}\sum_{h=0}^{H_i} \phi(s^i_h,a^i_h),
\end{align*}
is an unbiased estimator of $\sigma^\pi$. We denote this sampling procedure as $\operatorname{Est}\sigma(\pi, B)$.
\subsection{Algorithm summary}
Combining the above steps, we present Algorithm \ref{alg:irl} for learning the reward in \eqref{eq:irl_problem}, which simultaneously updates the policy with stochastic soft policy iteration and the reward parameter via stochastic projected gradient descent.
\begin{algorithm}[h]
    \SetAlgoVlined
    \caption{Primal-dual IRL algorithm}
    \label{alg:irl}
    \textbf{Input:} Reward learning rate $\eta_w>0$ and batch size $B$.\\
    Initialize $\pi^0\in\Delta_{\Acal}^{\Scal}$ and $w^0=0$.\\
    Estimate $\hat{\sigma}^E = \frac{1}{N}\sum_{i=0}^{N-1}\sum_{h=0}^{H-1} \gamma^h \phi(s_h^i, a_h^i)$ from $\mathcal{D}^E$.\\
    \For{$t\gets0$ \KwTo $T-1$}{
        $r^{t} = \ip{w^{t}}{\phi(\cdot,\cdot)}.$      
        
        \tcp{Estimate values:}
        \For{$s\in\Scal, a\in\Acal$}{
          $\hat Q^{\pi^t}_{r^t}(s,a) = \operatorname{Est}Q(s,a,\pi^t,r_t,B)$.}
        $\sigma^{\pi^t} = \operatorname{Est}\sigma(\pi^t, B)$.
        
        \tcp{Update policy and reward:}
        $\pi^{t+1}(\cdot|s)\propto\exp \br{\hat{Q}^{\pi^{t}}_{r^{t}}(s,\cdot)/\tau}$.
        
        $w^{t+1}=\mathcal{P}_W \left( w^{t} - \eta_w \left(\hat\sigma^{\pi^t}-\hat \sigma^{E} \right)\right)$.
    }
    \textbf{Output: }Reward $\bar r = r_{\bar w}$, with $\bar{w} = \frac{1}{T}\sum_{t=0}^{T-1} w^{t}$.
\end{algorithm}
Note that Algorithm~\ref{alg:irl} updates the policy and the reward parameters in a single loop. That is, we do not have to solve an RL problem at every reward step, but we only employ a single approximate policy iteration update. This is in contrast to the algorithm proposed by \cite{Syed2007}. Moreover, unlike the algorithm provided by \cite{Zeng2022}, which requires the exact state-action value function \eqref{Q} for the policy update and an infinitely long trajectory to estimate the feature expectation \eqref{featureexpectation}, Algorithm~\ref{alg:irl} is both model-free and implementable. It requires only a finite expected number of samples from the MDP per iteration. 

In the next section, we show that Algorithm~\ref{alg:irl} enjoys strong guarantees for the recovered reward.
\section{Convergence analysis}
In Section \ref{rewardsection}, we prove that our algorithm is guaranteed to recover a reward for which the expert is $\varepsilon$-optimal using an expected number of $\mathcal{O}(1/\varepsilon^{2})$ samples of the MDP, as detailed in Theorem \ref{theorem:stoch_irl_conv_proof}. Then, in Section \ref{policysection}, we establish that with an expected number of $\mathcal{O}(1/\varepsilon^{4})$ samples, the optimal policy corresponding to the recovered reward is $\varepsilon$-close to the expert policy in total variation distance, as stated in Corollary \ref{lemma:convinoccmeasure}. Furthermore, we show that the total variation distance is a stronger metric for measuring policy differences compared to the metrics used in \cite{Syed2007, viano2022proximal,Shani2021,Liu2022}.

\subsection{Reward convergence}\label{rewardsection}
To establish guarantees for the recovered reward $\bar r$, we require two assumptions. First, Assumption \ref{asm:distrib_bound} below ensures that the policy iterates of Algorithm~\ref{alg:irl} sufficiently explore the state space.
\begin{assumption}\label{asm:distrib_bound}
The distribution mismatch coefficient, defined by \(\vartheta := \max_{0\le t\le T-1} \max_{s\in\Scal} \nu^{\pi_{r^t}^{*}}(s)/\nu^{\pi^{t}}(s) \), is bounded from above.
\end{assumption}
Similar assumptions on the distribution mismatch coefficient are used in the prior literature \cite{Agarwal2020,Mei2020,ying2022dual}. Since $\nu^{\pi^{t}}(s)\geq  (1-\gamma)\nu_0(s)$, Assumption~\ref{asm:distrib_bound} is satisfied if the initial distribution $\nu_0$ is bounded away from zero.

Second, to show that the expert is approximately optimal for the recovered reward, we need the following approximate realizability assumption that quantifies the best-case optimality of the expert within our reward class $\Rcal$.
\begin{assumption} \label{asm:realizability}
    There exists $\varepsilon_{\text{real}}\ge 0$ such that the expert policy $\pi^E$ is $\varepsilon_{\text{real}}$-optimal for some reward $r\in\mathcal{R}$. That is,
     \begin{align}
        \min_{w\in W}\max_{\pi} L(\pi,w)=\min_{r\in\mathcal{R}} J^*_{r} - J^{\pi^E}_{r} \leq \varepsilon_\text{real}.   
    \end{align}    
\end{assumption}
 Assumption \ref{asm:realizability} has been introduced by \cite{Syed2007}. The realizability error $\varepsilon_{\text{real}}$ can be reduced by increasing the number of features $k$ and the diameter of $W$ \cite{Syed2007}. Next, we are ready to state our main convergence result.
\begin{theorem}\label{theorem:stoch_irl_conv_proof}
Suppose Assumptions \ref{asm:distrib_bound} and \ref{asm:realizability} hold, and let $\eta_w = \frac{(1-\gamma)}{\sqrt{kT}\|\phi\|_{\infty}}$. The expert satisfies the following optimality guarantee for the reward $\bar{r}$ returned by Algorithm~\ref{alg:irl}:
\begin{align}
    \E \Big[ J^*_{\bar r} - J^{\pi^E}_{\bar r} \Big]
    &\leq \varepsilon_\text{real}+\mathcal{O}\left(\gamma^H\right) +\Ocal\br{1/\sqrt{T}}.
\end{align}
Here, the expectation is taken with respect to all the randomness in Algorithm \ref{alg:irl}. Moreover, to recover a reward for which the expert is $(\varepsilon+\varepsilon_{\text{real}})$-optimal we require the length of the expert trajectories to be $H=\Ocal(\log\varepsilon^{-1})$ and we need an expected number of $\mathcal{O}(1/\varepsilon^{2})$ samples from the MDP.
\end{theorem}
The proof of Theorem~\ref{theorem:stoch_irl_conv_proof} is based on two key ingredients. First, Lemma~\ref{policyconverge} below shows that the policy iterates $\pi^t$ converge to the optimal policy for the reward iterates $r^t$ defined in Algorithm~\ref{alg:irl}. 
\begin{replemma}{policyconverge}[]
Suppose Assumption \ref{asm:distrib_bound} holds and let $\eta_w = \frac{1-\gamma}{\sqrt{kT}\|\phi\|_{\infty}}$. We can bound the suboptimality of the policy iterates of Algorithm \ref{alg:irl} by
\begin{align}
    \mathbb{E} \left[\max_\pi L(\pi,w^t)-  L(\pi^t,w^t)\right] \leq \Ocal\br{1/\sqrt{T}}.
\end{align}
\end{replemma}
The above lemma shows that if we control the changes of the reward $r^t$ by setting the reward learning rate to $\eta_w = \Theta(1/\sqrt{T})$, the value of the policy \(\pi^t\) converges to the optimal value under \(r^t\) at the same speed \(\Ocal(1/\sqrt{T})\). We provide the proof in Appendix \ref{AppendixA}. 

The second ingredient required for the proof of Theorem~\ref{theorem:stoch_irl_conv_proof} is the following lemma that shows that our algorithm has sublinear regret with respect to the reward.
\begin{replemma}{rewardconverge}[Stochastic online gradient descent regret]
If we set the learning rate to $\eta_w = \frac{1-\gamma}{\sqrt{kT}\|\phi\|_{\infty}}$, we have
    \begin{align*}
    \E\left[ \sum^{T-1}_{t=0} L(\pi^t, w^t)\right]\le\; & \E\left[\min_{w\in W} \sum^{T-1}_{t=0} L(\pi^t, w)\right]+ \mathcal{O}\left(\sqrt{T}\right).
\end{align*}
\end{replemma}
For exact gradient information, Lemma~\ref{rewardconverge} is a well-known result in online convex optimization \cite{Zinkevich03}. Since we use an unbiased gradient estimator, this result can be easily extended to our case. We provide a proof adapted to our setting in Appendix \ref{AppendixA}. Equipped with the above two lemmas, we are now ready to prove Theorem \ref{theorem:stoch_irl_conv_proof}.
\begin{proof}[Proof of Theorem \ref{theorem:stoch_irl_conv_proof}]
We first upper bound \(J^*_{\bar r} - J^{\pi^E}_{\bar r}\) as follows:
    \begin{align}
       &\E \left[ J^*_{\bar r} - J^{\pi^E}_{\bar r} \right] \nonumber\\
       =&\E \left[\max_{\pi}\sum_{t=0}^{T-1}\frac{L(\pi,w^t)+\langle w^t,\sigma^{\pi^E}-\hat \sigma^E\rangle}{T} \right] \nonumber\\
       \stackrel{(i)}{\le}& \E \left[\max_{\pi}\sum_{t=0}^{T-1}\frac{L(\pi,w^t)}{T} \right]+\mathcal{O}(\gamma^H)\nonumber\\
       \stackrel{(ii)}{\le}& \E \left[\sum_{t=0}^{T-1}\frac{\max_{\pi}L(\pi,w^t) }{T} \right]+\mathcal{O}(\gamma^H)\nonumber\\
       \stackrel{(iii)}{\le}& \E \left[\sum_{t=0}^{T-1}\frac{L(\pi^t,w^t) }{T} \right]+\mathcal{O}\left(\gamma^H\right) +\mathcal{O}\left({1}/{\sqrt{T}}\right)\nonumber\\
       \stackrel{(iv)}{\le}& \E \left[\min_{w\in W}\frac{\sum_{t=0}^{T-1}L(\pi^t,w) }{T} \right]+\mathcal{O}\left(\gamma^H\right) +\mathcal{O}\left({1}/{\sqrt{T}}\right).
       \label{step1}
    \end{align}
Here, $(i)$ follows from the truncation error of the empirical expert feature expectation, as detailed in Lemma \ref{Unbiased_A_sampler}, and $(ii)$ holds since 
 $\max_{\pi}L(\pi,\cdot)$ is a pointwise maximum of affine functions and therefore convex. Moreover, in $(iii)$ and $(iv)$ we used Lemma~\ref{policyconverge} and \ref{rewardconverge}, respectively. 
Next, notice that we can express $L(\pi^t,w)$ as a function of the state-action occupancy measure, $\mu^\pi(s,a):=\nu^\pi(s)\pi(a|s)$. In particular, we can rewrite $L(\pi,w) = \bar L(\mu^{\pi},w)$, with
\begin{align*}
    \bar L(\mu,w) &:= \langle r_w-\tau \log\pi^{\mu}, \mu \rangle-\langle w, \hat{\sigma}^E \rangle-\tau \Hh(\pi^E),
\end{align*}
where $\pi^\mu$ denotes the policy induced by $\mu$.
It can be shown that $\bar L(\cdot,w)$ is concave \cite{Schlaginhaufen2023}. Therefore, if $\bar \mu:=\frac{1}{T}\sum_{t=0}^{T-1}\mu^{\pi^t}$ and $\bar \pi:= \pi^{\bar \mu}$, we have by Jensen's inequality that
\begin{align}
   \frac{1}{T} \sum_{t=0}^{T-1}L(\pi^t,w)=\frac{1}{T} \sum_{t=0}^{T-1}\bar L(\mu^{\pi^t},w)\le \bar L\left(\bar \mu,w\right)= L(\bar\pi, w).
\end{align}
Plugging this back into \eqref{step1}, we have
\begin{align*}
  \E \left[ J^*_{\bar r} - J^{\pi^E}_{\bar r} \right]\le& \E \left[\min_{w\in W}L(\bar \pi,w) \right]+\mathcal{O}\left(\gamma^H\right) +\mathcal{O}\left({1}/{\sqrt{T}}\right)\\
    {\le}& \varepsilon_{\text{real}}+\mathcal{O}\left(\gamma^H\right) +\mathcal{O}\left({1}/{\sqrt{T}}\right),
\end{align*}
where the last step follows by Assumption \ref{asm:realizability}. To find the reward for which the expert is $(\varepsilon+\varepsilon_{\text{real}})$-optimal, we need to set $H=\mathcal{O}\left(\log\varepsilon^{-1}\right)$ and $T=\mathcal{O}(1/\varepsilon^{2})$. Since each trajectory in $\operatorname{Est}Q(s,a,\pi,r, B)$ and $\operatorname{Est}\sigma(\pi^t, B)$ has an expected length of \(1/(1-\gamma)\), the expected total number of samples used by Algorithm \ref{alg:irl} is $2BT/\br{1-\gamma} =\mathcal{O}(1/\varepsilon^{2})$.
\end{proof}
\subsection{Policy convergence}\label{policysection}
In Theorem \ref{theorem:stoch_irl_conv_proof}, we quantified the optimality of the expert for the recovered reward $\bar{r}$. In this section, we show that the optimal policy corresponding to the recovered reward is also close to the expert policy. We first introduce a total variation metric for measuring the distance to the expert policy
\begin{equation}\label{eq:tv_metric}
    \max_s \| \pi^E(\cdot|s) - \pi(\cdot|s) \|_{\text{TV}}.
\end{equation}
We define $\pi$ to be $\varepsilon$-close to the expert policy if the total variation metric described above is bounded by $\varepsilon$. To ensure convergence in the policy, we need the following additional assumption.
\begin{assumption} \label{asm:distrib_bound_expert}
The state occupancy measure $\nu^{\pi^E}(s)$ generated by the expert satisfies $\vartheta^E := \min_{s\in\mathcal{S}} \nu^{\pi^E}(s)>0$. 
\end{assumption}
The above assumption ensures that the expert sufficiently explores the state space of the MDP. Similar to Assumption \ref{asm:distrib_bound}, it is satisfied when the initial distribution $\nu_0$ is bounded away from zero. Under Assumption~\ref{asm:distrib_bound_expert}, we have the following convergence guarantee for the optimal policy corresponding to the recovered reward $\bar r$.
\begin{corollary}\label{lemma:convinoccmeasure}
Suppose Assumptions \ref{asm:distrib_bound}, \ref{asm:realizability}, and \ref{asm:distrib_bound_expert} hold, and let $\eta_w = \frac{1-\gamma}{\sqrt{kT}\|\phi\|_{\infty}}$. Algorithm \ref{alg:irl} requires an expected number of $\Ocal(1/\varepsilon^{4})$ samples to ensure that the optimal policy corresponding to the recovered reward $\bar r$ is $(\varepsilon +\sqrt{\varepsilon_{\text{real}}})$-close to the expert policy.
\end{corollary}
\begin{proof}
The result follows from
    \begin{align}
        J^*_{\bar r} -J^{\pi^E}_{\bar r} \stackrel{(i)}{=}& \frac{\tau}{1-\gamma} \E_{s\sim\nu^E}\Big[\text{KL}(\pi^E(\cdot|s)||\pi^*_{r^{\bar{w}}}(\cdot|s))\Big]\\
        =& \frac{\tau}{1-\gamma} \sum_{s\in\mathcal{S}} \nu^E(s) \text{KL}(\pi^E(\cdot|s)||\pi^*_{r^{\bar{w}}}(\cdot|s))  \\
        \stackrel{(ii)}{\geq} &\frac{2\tau}{1-\gamma} \sum_{s\in\mathcal{S}} \nu^E(s) \|\pi^E(\cdot|s) - \pi^*_{r^{\bar{w}}}(\cdot|s)\|_{\text{TV}}^2\\
        \stackrel{(iii)}{\geq}& \frac{2\tau \vartheta^E}{1-\gamma}   \max_s\|\pi^E(\cdot|s) - \pi^*_{r^{\bar{w}}}(\cdot|s)\|_{\text{TV}}^2,
    \end{align}
where the equality $(i)$ follows from the soft suboptimality Lemma \cite[Lemma 26]{Mei2020}, $(ii)$ holds by Pinsker's inequality, and $(iii)$ by Assumption \ref{asm:distrib_bound_expert}. Therefore, to ensure $\max_s\|\pi^E(\cdot|s) - \pi^*_{r^{\bar{w}}}(\cdot|s)\|_{\text{TV}}\leq\varepsilon +\sqrt{\varepsilon_{\text{real}}}$, we need $J^*_{\bar r} -J^{\pi^E}_{\bar r}$ to be upper bounded by $\Omega\left(\varepsilon^2+\varepsilon_{\text{real}}\right)$. By Theorem \ref{theorem:stoch_irl_conv_proof}, we need an expected number of $\mathcal{O}(1/\varepsilon^{4})$ samples in total.
\end{proof}

In Corollary \ref{lemma:convinoccmeasure}, we show that the optimal policy corresponding to the reward recovered by Algorithm~\ref{alg:irl} converges to the expert's policy in the total variation metric~\eqref{eq:tv_metric}. In the following, we demonstrate that \eqref{eq:tv_metric} is a stronger metric for measuring policy convergence compared to the metrics used by \cite{Syed2007, viano2022proximal, Shani2021, Liu2022}. In particular, the authors of \cite{Syed2007} prove convergence in the metric
\begin{align}
      \langle w_{\text{true}}, \sigma^{\pi} - \sigma^{\pi^E} \rangle , \label{metric1}
\end{align}
where $w_{\text{true}}\in W$ is an unknown true reward parameter. Moreover, \cite{viano2022proximal, Shani2021, Liu2022} 
provide their convergence guarantees in terms of the \emph{integral probability metric} \cite{muller1997integral}
\begin{align}
    \hspace{-0.4cm}\delta_{\Rcal}(\mu^\pi, \mu^{\pi^E}):= \max_{r\in\Rcal}\ip{r}{\mu^{\pi}-\mu^{\pi^E}} =\max_{w\in W} \langle w, \sigma^{\pi} - \sigma^{\pi^E} \rangle, \label{metric2}
\end{align}
between the state-action occupancy measures $\mu^{\pi}$ and $\mu^{\pi^E}$. The metric \eqref{metric2} measures the worst-case difference in the unregularized expected value between the recovered policy and the expert policy. It is easy to see that the integral probability metric \eqref{metric2} is stronger compared to the metric \eqref{metric1}. In the following proposition, we demonstrate that the total variation metric \eqref{eq:tv_metric} is a stronger metric for measuring policy convergence compared to the integral probability metric \eqref{metric2}.
\begin{proposition}\label{metriccompare}
1) If the policy $\pi$ is $\varepsilon$-close to the expert policy in the total variation metric, then it is also $\varepsilon$-close to the expert policy in the integral probability metric.

2) Convergence in the integral probability metric does not imply convergence in the total variation metric.
    
\end{proposition}
\begin{proof}
To prove $1)$, we bound \eqref{metric2} as follows:
\begin{align*}
    \max_{w\in W} \langle w, \sigma^{\pi} - \sigma^{\pi^E} \rangle &\stackrel{(i)}{\le} \max_{w\in W} \|w\|_1 \|\sigma^{\pi} - \sigma^{\pi^E}\|_\infty \nonumber \\
    &\stackrel{(ii)}{\le} \|\sigma^{\pi} - \sigma^{\pi^E}\|_\infty \nonumber \\
    &\stackrel{(iii)}{\le}\frac{2\|\phi\|_\infty}{(1-\gamma)^2} \max_s \|\pi^E(\cdot|s) - \pi(\cdot|s) \|_{\text{TV}},
\end{align*}
where $(i)$ follows from Hölder's inequality, $(ii)$ holds because $\|w\|_1 \leq 1$, and $(iii)$ uses the Lipschitz continuity of $\sigma^{\pi}$ with respect to $\pi$, as shown in Lemma \ref{lem:lip_feat}.

To prove $2)$, we consider a one-state MDP, as illustrated in Figure~\ref{mdpex}, where the policy is optimal in the integral probability metric \eqref{metric2} but is far from the expert policy in the total variation metric.

Consider an MDP with a single state \(\Scal=\bc{s_1}\) and two actions \(\mathcal{A} = \{a_1, a_2\}\). We let the feature vector be a scalar constant $\phi(s,a) = 1$, and we consider the expert policy $\pi^E(a_1 | s_1) = 1/2$ and the policy $\pi(a_2 | s_1) = 1$. Then,
\begin{align}
&\max_{w \in W} \langle w, \sigma^{\pi} - \sigma^{\pi^E} \rangle = 0, \quad \max_s \| \pi^E(\cdot | s) - \pi(\cdot | s) \|_{\text{TV}} = \frac{1}{2}.
\end{align}
\vspace{-0.6cm}
\begin{figure}[H]
\centering
\begin{center}
\begin{tikzpicture}[auto,node distance=2cm,>=latex,font=\small]
\tikzstyle{round}=[thick,draw=black,circle]
    \node[round] (s0) {$s_1$};
    \path[draw=black,solid,line width=0.25mm,fill=black] (s0) edge [loop left] node {$a_1$} (s0);
    \path[draw=black,solid,line width=0.25mm,fill=black] (s0) edge [loop right] node {$a_2$} (s0);
\end{tikzpicture}
\end{center}
\caption{One-state MDP}
\label{mdpex}
\end{figure}
\end{proof}
As shown by Proposition \ref{metriccompare}, convergence in the total variation metric \eqref{eq:tv_metric} is stronger than convergence in the integral probability metric \eqref{metric2}. This is because the total variation metric directly measures the difference between policies, while the integral probability metric measures the difference in their unregularized expected value, ignoring the entropy regularization values. In the above example, the expert policy $\pi^E$ has maximum entropy and is therefore realizable with $\varepsilon_{\text{real}}=0$. Therefore, Algorithm~\ref{alg:irl} is guaranteed to recover a reward with a corresponding optimal policy that is $\varepsilon$-close to the expert policy in the total variation metric. However, in the case of a large realizability error $\varepsilon_{\text{real}}$, Corollary~\ref{lemma:convinoccmeasure} fails to provide such strong guarantees. Hence, compared with \cite{Syed2007, viano2022proximal, Shani2021, Liu2022}, we can get stronger convergence guarantees for the optimal policy corresponding to the recovered reward if the expert is realizable with a small realizability error $\varepsilon_{\text{real}}$.

\section{Discussion and conclusion}
We proposed a model-free single-loop algorithm to tackle the entropy-regularized IRL problem. Our algorithm simultaneously updates the policy using stochastic soft policy iteration and the reward parameters via stochastic projected gradient descent. We provided theoretical guarantees for the recovered reward and characterized the algorithm's sample complexity. Moreover, we demonstrated that the optimal policy under the recovered reward is close to the expert policy, measured using the total variation metric. Furthermore, we showed that this metric is stronger than the metrics used by \cite{Syed2007, viano2022proximal, Shani2021, Liu2022}. 

Since our proposed algorithm uses a stochastic soft policy iteration update for the policy, it requires re-estimating the state-action values for all state-action pairs at each time step. This may be infeasible for large state and action spaces and may require a large batch size to control the variance of the policy iterates. This highlights a limitation in our theoretical results, as we have only demonstrated convergence in expectation without providing a high probability bound. Therefore, a potential avenue for future research involves redesigning the policy update steps to establish convergence with a high probability bound and validating the algorithm on benchmarks or real-world scenarios.
\bibliographystyle{abbrv}
\bibliography{paper}
\appendix
\section{Proof of main lemmas}\label{AppendixA}
\begin{lemma}[{Property of the estimators}]\label{Unbiased_A_sampler}
The estimators $\hat{Q}^{\pi}_{r}(s,a)$ and $\hat{\sigma}^{\pi}$ returned by $\operatorname{Est}Q(s,a,\pi, r,B)$ and $\operatorname{Est}\sigma(\pi, B)$ respectively, are unbiased, meaning
\begin{align*}
    \mathbb{E}\left[\hat{Q}^{\pi}_{r}(s,a)\right] &= {Q}^{\pi}_{r}(s,a), \,\mathbb{E}\left[\hat{\sigma}^{\pi}\right] = \sigma^{\pi}.
\end{align*}
For the estimator $\hat{\sigma}^{E}$, we have, for all $w\in W$,
\begin{align*}
    \langle\sigma^{\pi^E},w\rangle- \frac{\gamma^H \|\phi\|_{\infty}}{1-\gamma} &\le \langle\mathbb{E}\left[\hat{\sigma}^{E}\right],w \rangle.
\end{align*}
For the reward gradient estimator $\hat \nabla_w L(\pi,w)$, we have
\begin{align}
    \E\|\hat \nabla_w L(\pi,w)\|_\infty\le \frac{2\|\phi\|_\infty}{1-\gamma},\,\E\|\hat \nabla_w L(\pi,w)\|_2^2\le \frac{6{k}\|\phi\|_{\infty}^2}{(1-\gamma)^2}.
\end{align}
Here, the expectations are with respect to the randomness of the corresponding sampling strategies.
\end{lemma}
\begin{proof}[Proof of Lemma \ref{Unbiased_A_sampler}]
From \cite[Assumption 6.3]{Agarwal2020}, we have $\mathbb{E}\left[\hat{Q}^{\pi}_{r}(s,a)\right] = {Q}^{\pi}_{r}(s,a)$ and $(s_H,a_H)$ with $H\sim \operatorname{Geom}(1-\gamma)$ is sampled from the state-action occupancy measure $\mu^{\pi}$. Therefore, we have
\begin{align*}
 \E\left[\hat\sigma^{\pi}\right]= \E\left[\dfrac{1}{B (1-\gamma)}\sum_{i=1}^{B}\sum_{h=0}^{H_i} \phi(s^i_h,a^i_h)\right]=\E_{(s^i_h,a^i_h)\sim \mu^{\pi}(s,a)}\left[\dfrac{1}{B }\sum_{i=1}^{B}\sum_{h=0}^{H_i} \phi(s^i_h,a^i_h)\right]=\sigma^{\pi}.
\end{align*}
For estimator $\hat{\sigma}^{E}$, we have
\begin{align*}
    \mathbb{E}\left[\hat{\sigma}^{E}\right] &=\E_{\pi^E}\left[\sum_{t=0}^{H-1}\gamma^t\phi(s,a)\right]\\
    &= \E_{\pi^E}\left[\sum_{t=0}^{\infty}\gamma^t\phi(s,a)\right]-\E_{\pi^E}\left[\sum_{t=H}^{\infty}\gamma^t\phi(s,a)\right]\\
    &= {\sigma}^{\pi^E}-\E_{\pi^E}\left[\sum_{t=H}^{\infty}\gamma^t\phi(s,a)\right].
\end{align*}
Taking the inner product with $w\in W$ on both sides of the above equality, we have
\begin{align*}
    \langle\mathbb{E}\left[\hat{\sigma}^{E}\right],w \rangle &=\langle\sigma^{\pi^E},w\rangle- \E_{\pi^E}\left[\sum_{t=H}^{\infty}\gamma^t\langle\phi(s,a),w\rangle\right]\\
    &\ge \langle\sigma^{\pi^E},w\rangle- \frac{\gamma^H \|\phi\|_{\infty}}{1-\gamma},
\end{align*}
where the last inequality follows from $\langle\phi(s,a),w\rangle\le\|\phi\|_{\infty} \|w\|_1\le \|\phi\|_{\infty}$.

For reward gradient estimator $\hat \nabla_w L(\pi,w)$, we have
\begin{align}
   \E\left\| \hat \nabla_w L(\pi,w)\right\|_\infty=\E\left\| \hat\sigma^{\pi}-\hat \sigma^{E} \right\|_\infty&=\E\left\|\dfrac{1}{B}\sum_{i=1}^{B}\sum_{h=0}^{H_i} \phi(s^i_h,a^i_h)-\frac{1}{N}\sum_{j=1}^{N}\sum_{h=0}^{H-1} \gamma^h \phi(s_h^j, a_h^j)\right\|_\infty\\
   &\le  \E\left\|\dfrac{1}{B }\sum_{i=1}^{B}\sum_{h=0}^{H_i} \phi(s^i_h,a^i_h)\right\|_\infty+\E\left\|\frac{1}{N}\sum_{i=1}^{N}\sum_{h=0}^{H-1} \gamma^h \phi(s_h^i, a_h^i)\right\|_\infty\\
   &\stackrel{(i)}{\le} {\|\phi\|_\infty}\E_{H\sim\operatorname{Geom}(1-\gamma)}H+\frac{\|\phi\|_\infty}{1-\gamma}\\
     &\le \frac{(1+\gamma)\|\phi\|_\infty}{1-\gamma}\le \frac{2\|\phi\|_\infty}{1-\gamma},
\end{align}
where $(i)$ we apply the expectation of $\operatorname{Geom}(1-\gamma)$ is $\frac{\gamma}{1-\gamma}$. Similarly, we can bound $\E\left\| \hat \nabla_w L(\pi,w)\right\|_2^2$ as follows:
\begin{align}
   \E\left\| \hat \nabla_w L(\pi,w)\right\|_2^2&=\E\left\|\dfrac{1}{B}\sum_{i=1}^{B}\sum_{h=0}^{H_i} \phi(s^i_h,a^i_h)-\frac{1}{N}\sum_{j=1}^{N}\sum_{h=0}^{H-1} \gamma^h \phi(s_h^j, a_h^j)\right\|_2^2\\
   &\le 2 \E\left\|\dfrac{1}{B }\sum_{i=1}^{B}\sum_{h=0}^{H_i} \phi(s^i_h,a^i_h)\right\|_2^2+2\E\left\|\frac{1}{N}\sum_{i=1}^{N}\sum_{h=0}^{H-1} \gamma^h \phi(s_h^i, a_h^i)\right\|_2^2\\
     &\stackrel{(i)}{\le} {2k\|\phi\|_\infty^2}\E_{H\sim\operatorname{Geo}(1-\gamma)}H^2+\frac{2k\|\phi\|_\infty^2}{(1-\gamma)^2}\\
     &\le \frac{2k(1+\gamma+\gamma^2)\|\phi\|_\infty^2}{(1-\gamma)^2}\le \frac{6k\|\phi\|_\infty^2}{(1-\gamma)^2},
\end{align}
where $(i)$ we apply the fist moment of $\operatorname{Geom}(1-\gamma)$ is $\frac{\gamma+\gamma^2}{(1-\gamma)^2}$.
\end{proof}
\repeatlemma{policyconverge}
\begin{proof}[Proof of Lemma \ref{policyconverge}]
We first upper bound the improvement of the suboptimality gap as follows:
\begin{align}   
     &\mathbb{E} \left[\left(\max_{\pi}L(\pi,w^{t+1})
    -L(\pi^{t+1},w^{t+1})\right)-
    \left(\max_{\pi}L(\pi,w^{t})
    - L(\pi^t,w^t)\right)\right]\\
     =&\mathbb{E} \left[\left(J^*_{r^{t+1}}
    -J^{\pi^{t+1}}_{r^{t+1}}\right)-
    \left(J^*_{r^{t}}
    - J^{\pi^{t}}_{r^{t}}\right)\right]\\
    \stackrel{(i)}{=}&
    \mathbb{E} \left[J^*_{r^{t+1}}
    - J^{\pi^{t+1}}_{r^{t+1}}
    - J^{\pi^{t+1}}_{r^{t}}
    + J^{\pi^{t+1}}_{r^{t}}
    - J^{*}_{r^{t}}
    + J^{\pi^{t}}_{r^{t}}\right]\\
    \le & 
    \mathbb{E} \left[\big| J^*_{r^{t+1}} - J^{*}_{r^{t}} \big| 
    + \big| J^{\pi^{t+1}}_{r^{t+1}}
    - J^{\pi^{t+1}}_{r^{t}} \big|
    - \big( J^{\pi^{t+1}}_{r^{t}}
    - J^{\pi^{t}}_{r^{t}} \big)\right]\\
    \stackrel{(ii)}{\leq} &
    \frac{2\eta_w\|\phi\|_{\infty}}{(1-\gamma)^2}
    - \mathbb{E} \left[ J^{\pi^{t+1}}_{r^{t}}
    - J^{\pi^{t}}_{r^{t}} \right]\\
    \stackrel{(iii)}{\leq} & 
    \frac{2\eta_w\|\phi\|_{\infty}}{(1-\gamma)^2}
    -  \frac{\tau}{1-\gamma} \mathbb{E}\left[
        \E_{s \sim \nu^{\pi^{t+1}}} \text{KL}(\pi^{t}(\cdot|s)||\pi^{t+1}(\cdot|s))\right]\\
    \leq  &  \frac{2\eta_w\|\phi\|_{\infty}}{(1-\gamma)^2}
    - \frac{\tau}{1-\gamma} \mathbb{E}\left[\br{\min_{s}\frac{\nu^{\pi^{t+1}}(s)}{\nu^{\pi_{r^{t}}^{*}}(s)}}
        \E_{s \sim \nu^{\pi_{r^{t}}^{*}}} \text{KL}(\pi^{t}(\cdot|s)||\pi^{t+1}(\cdot|s))\right]\\
    \stackrel{(iv)}{\leq}  &  \frac{2\eta_w\|\phi\|_{\infty}}{(1-\gamma)^2}
    -  \frac{\tau}{\vartheta(1-\gamma)} \mathbb{E}
        \left[\E_{s \sim \nu^{\pi_{r^{t}}^{*}}} \text{KL}(\pi^{t}(\cdot|s)||\pi^{t+1}(\cdot|s))\right]\\
    \stackrel{(v)}{\leq}  &  \frac{2\eta_w\|\phi\|_{\infty}}{(1-\gamma)^2}
    -  \frac{1}{\vartheta} \mathbb{E}\left[J^*_{r^{t}} - J^{\pi^{t}}_{r^{t}}\right]\\
    =  &  \frac{2\eta_w\|\phi\|_{\infty}}{(1-\gamma)^2}
    -  \frac{1}{\vartheta} \mathbb{E}\left[\max_{\pi}L(\pi,w^{t})-L(\pi^{t},w^{t})\right]
\end{align}
where $(i)$ holds by adding and subtracting $J^{\pi^{t+1}}_{r^{t}}$,  $(ii)$ holds by Lemma \ref{lemma:j_lipschitz_reward_fixed_pi}, $(iii)$ holds by Lemma \ref{lemma:npg_ascent}, $(iv)$ holds by Assumption \ref{asm:distrib_bound} and $(v)$ holds by Lemma \ref{lemma:policy_step_scales_subopt}. Rearrange the above inequality, we have
\begin{align}   
    \mathbb{E} \left[\max_{\pi}L(\pi,w^{t+1})
    -L(\pi^{t+1},w^{t+1})\right]
    \leq  \frac{2\eta_w\|\phi\|_{\infty}}{(1-\gamma)^2}
    +\left(1-  \frac{1}{\vartheta} \right)\mathbb{E}\left[\max_{\pi}L(\pi,w^{t})-L(\pi^{t},w^{t})\right].
\end{align}
Recursively applying the above inequality, we obtain
\begin{align}   
    \mathbb{E} \left[\max_{\pi}L(\pi,w^{t})-L(\pi^{t},w^{t})\right]
    \leq &\frac{2\eta_w\vartheta\|\phi\|_{\infty}}{(1-\gamma)^2}
    +\left(1 -  \frac{1}{\vartheta} \right)^t\left(\max_{\pi}L(\pi,w^{0})-L(\pi^{0},w^{0})\right)\\
    = &\frac{2\vartheta}{\sqrt{kT}(1-\gamma)},
\end{align}
where we apply $\eta_w=\frac{1-\gamma}{\sqrt{kT}\|\phi\|_\infty}$ and $\max_{\pi}L(\pi,w^{0})-L(\pi^{0},w^{0}) = 0$ since $r^{(0)} = 0$ in the last step.
\end{proof}
\repeatlemma{rewardconverge}
\begin{proof}
    We start from the projected descent step
    \begin{align*}
         \|w^{(t+1)}-w^*\|_2^2 &= \left\|\mathcal{P}_{W}\left(w^{t}-\eta_w\hat \nabla_w L(\pi^t,w^t)\right)-w^*\right\|_2^2\\
        &\stackrel{(i)}{\leq}  \left\|w^{t}-\eta_w\hat \nabla_w L(\pi^t,w^t)-w^*\right\|_2^2 \\
        &=  \| w^{t} -w^* \|_2^2 
        - 2\eta_w \langle\hat \nabla_w L(\pi^t,w^t), w^{t} - w^* \rangle
        + \eta_w^2 \|  \hat \nabla_w L(\pi^t,w^t)\|_2^2,
    \end{align*}
    where $(i)$ holds by the non expansiveness of projection. Therefore
    \begin{align}
        \langle\hat \nabla_w L(\pi^t,w^t), w^{t} - w^* \rangle
        \leq \frac{ \| w^{t} -w^* \|_2^2   -  \| w^{t+1} -w^* \|_2^2  }{2 \eta_w}
        +  \frac{\eta_w}{2} \| \hat\nabla_w L(\pi^t,w^t)\|_2^2. 
    \end{align}
    Taking the expectation on both sides, we have 
    \begin{equation}\label{eq:cvx_bound_reg}
        \E\left[\left\langle\nabla_w L(\pi^t,w^t), w^{t} - w^*\right \rangle\right]
        \leq \E\left[\frac{ \| w^{t} -w^* \|_2^2   -  \| w^{t+1} -w^* \|_2^2  }{2 \eta_w}
        +  \frac{3k\eta_w\|\phi\|_{\infty}^2}{(1-\gamma)^2}\right],
    \end{equation}
    where we use Lemma \ref{Unbiased_A_sampler}. Let $w^*$ be the following optimizer:
    \begin{align}
        w^*=\arg\min_{w \in W}\sum^{T-1}_{t=0}L(\pi^t,w).
    \end{align}
    We have
    \begin{align}
        \E\left[ \sum^{T-1}_{t=0} L(\pi^t, w^t)-\min_{w\in W} \sum^{T-1}_{t=0} L(\pi^t, w)\right]&=\E\left[ \sum^{T-1}_{t=0} L(\pi^t, w^t)-\sum^{T-1}_{t=0} L(\pi^t, w^*)\right]\\
        &\stackrel{(i)}{\leq}\E\left[\sum^{T-1}_{t=0} \left\langle \nabla_w L(\pi^t,w^t), w^t - w^{*} \right\rangle\right]\\
        &\stackrel{(ii)}{\leq}\E\left[\sum^{T-1}_{t=0}\frac{\| w^{t}-w^{*} \|_2^2  - \| w^{t+1} -w^{*} \|_2^2 }{2 \eta_w} + \frac{3kT\eta_w\|\phi\|_{\infty}^2}{(1-\gamma)^2}\right]\\
        &\leq\E\left[\frac{\| w^{T-1}-w^{*} \|_2^2 }{2 \eta_w} \right]+ \frac{3kT\eta_w\|\phi\|_{\infty}^2}{(1-\gamma)^2}\\
        &\leq\frac{1}{\eta_w}+ \frac{3kT\eta_w\|\phi\|_{\infty}^2}{(1-\gamma)^2}\\
        &\stackrel{(iii)}{\le } \frac{3\sqrt{2kT}\|\phi\|_{\infty}}{1-\gamma},
    \end{align}
    where $(i)$ holds by convexity, $(ii)$ by \eqref{eq:cvx_bound_reg}, and $(iii)$ by plugging in $\eta_w=\frac{1-\gamma}{\sqrt{2kT}\|\phi\|_{\infty}}$. 
\end{proof}
\section{Supporting lemmas}
\begin{lemma}[Soft suboptimality {\cite[Lemma 26]{Mei2020}}] \label{lemma:soft_suboptimality}
    For any policy $\pi \in \Delta_{\Acal}^{\Scal}$ and reward $ r  \in \mathcal{R}$, we have
    \begin{align}
         J^*_r - J^\pi_r = \frac{\tau}{1-\gamma} \E_{s\sim\nu^\pi}\text{KL}(\pi(\cdot|s)||\pi^*(\cdot|s)).
    \end{align}    
\end{lemma}
\begin{lemma}[Soft performance difference \cite{Lan2021}]
\label{lemma:pd_lemma}
For any two policies $\pi,\pi'\in\Delta_{\Acal}^{\Scal}$, we have
\begin{align}
    J_r^{\pi} - J_r^{\pi'} = \dfrac{1}{1-\gamma}\br{\E_{(s,a)\sim\mu^\pi}\bs{A_r^{\pi'}(s,a)} - \tau \E_{s\sim\nu^\pi}\bs{\kl{\pi(\cdot|s)}{\pi'(\cdot|s)}}},
\end{align}
where $A_r^{\pi'}(s,a) := Q_r^{\pi'}(s,a) - V_r^{\pi'}(s) - \tau \log \pi'(a|s)$ is the advantage function.
\end{lemma}
\begin{lemma}[Performance improvement for the policy]\label{lemma:npg_ascent}
    \begin{align}
        \mathbb{E}\left[J^{\pi^{t+1}}_{r^t} - J^{\pi^{t}}_{r^t}\right]
        = \frac{\tau}{1-\gamma} \mathbb{E}\left[
        \E_{s \sim \nu^{\pi^{t+1}}}\bs{\kl{\pi^{t}(\cdot|s)}{\pi^{t+1}(\cdot|s)}} \right]
    \end{align}
\end{lemma}
\noindent This result extends \cite[Lemma 1]{Cen2022} to the stochastic setting.
\begin{proof}
From the soft value iteration update,
    \begin{align}
        \pi^{t+1} (a|s)
        =  \frac{1}{Z^{t}(s)}\exp \big( \hat Q^{\pi^t}_{r^t}(s,a) / \tau \big)     = \frac{1}{Z^{t}(s)}\exp \Bigg( \frac{Q^{\pi^t}_{r^t}(s,a) + \Delta^{t}(s,a)}{\tau} \Bigg),
    \end{align}
    where $\Delta^{t}(s,a):=\hat Q^{\pi^t}_{r^t}(s,a)- Q^{\pi^t}_{r^t}(s,a)$. It follows that
    \begin{align}\label{eq:spi_log_update}
         {\tau}\log Z^{t}(s) =& Q_{r^t}^{\pi^t}(s,a) + \Delta^{t}(s,a) - {\tau}\log {\pi^{t+1} (a|s)}.
    \end{align}
Let 
Using Lemma~\ref{lemma:pd_lemma}, we have that
\begin{align}
    \E\bs{J^{\pi^{t+1}}_{r^t} - J^{\pi^{t}}_{r^t}} &= \dfrac{1}{1-\gamma}\E\bs{\E_{(s,a)\sim\mu^{\pi^{t+1}}}\bs{A_{r^t}^{\pi^t}(s,a)} - \tau \E_{s\sim\nu^{\pi^{t+1}}}\bs{\kl{\pi^{t+1}(\cdot|s)}{\pi^{t}(\cdot|s)}}}\\
    &\stackrel{(i)}{=}\dfrac{1}{1-\gamma}\E\left[\E_{(s,a)\sim\mu^{\pi^{t+1}}}\bs{Q_{r^t}^{\pi^t}(s,a) - V_{r^t}^{\pi^t}(s) - \tau \log \pi^{t+1}(a|s)} \right]\\
    &\stackrel{(ii)}{=}\dfrac{1}{1-\gamma}\E\left[\E_{s\sim\nu^{\pi^{t+1}}}\E_{a\sim\pi^{t+1}(\cdot|s)}\bs{{\tau}\log Z^{t}(s) - V_{r^t}^{\pi^t}(s) - \Delta^{t}(s,a)}\right]\\
    &\stackrel{(iii)}{=}\dfrac{1}{1-\gamma}\E\left[\E_{s\sim\nu^{\pi^{t+1}}}\E_{a\sim\pi^{t}(\cdot|s)}\bs{{\tau}\log Z^{t}(s) - V_{r^t}^{\pi^t}(s)}\right]\\
    &\stackrel{(iv)}{=}\dfrac{1}{1-\gamma}\E\left[\E_{s\sim\nu^{\pi^{t+1}}}\E_{a\sim\pi^{t}(\cdot|s)}\bs{Q_{r^t}^{\pi^t}(s,a) + \Delta^{t}(s,a) - {\tau}\log {\pi^{t+1} (a|s)} - V_{r^t}^{\pi^t}(s)}\right]\\
    &\stackrel{(v)}{=}\dfrac{1}{1-\gamma}\E\left[\E_{s\sim\nu^{\pi^{t+1}}}\E_{a\sim\pi^{t}(\cdot|s)}\bs{A_{r^t}^{\pi^t}(s,a) + \Delta^{t}(s,a) + {\tau}\log \dfrac{\pi^{t} (a|s)}{\pi^{t+1} (a|s)}}\right]\\
    &\stackrel{(vi)}{=}\dfrac{\tau}{1-\gamma}\E\left[\E_{s\sim\nu^{\pi^{t+1}}}\bs{\kl{\pi^{t} (\cdot|s)}{\pi^{t+1} (\cdot|s)}}\right].
\end{align}
Here, we use the definition of the advantage in $(i)$ and \eqref{eq:spi_log_update} in $(ii)$. In $(iii)$ we use 
\begin{equation}\label{eq:unbiasedQ}
    \E\bs{\E_{s\sim\nu^{\pi^{t+1}}}\E_{a\sim\pi^{t}(\cdot|s)}\bs{ \Delta^t(s,a)}} = \E\bs{\E_{s\sim\nu^{\pi^{t+1}}}\E_{a\sim\pi^{t}(\cdot|s)}\bs{\E\bs{\Delta^t(s,a)|\pi^t, r^t}}} = 0,
\end{equation}
where the last equality follows from Lemma~\ref{Unbiased_A_sampler}. In $(iv)$ we again plug in \eqref{eq:spi_log_update}. Finally, $(v)$ follows from rearranging and $(vi)$ from Equation \eqref{eq:unbiasedQ} and $\E_{a\sim\pi^{t}(\cdot|s)}\bs{A_{r^t}^{\pi^t}(s,a)}=0$.
\end{proof} 
\begin{lemma}[suboptimality gap for policy]\label{lemma:policy_step_scales_subopt}
For any iterates $r^t$ and $\pi^t$ generated by Algorithm \ref{alg:irl}, we have
    \begin{align}
       \mathbb{E}\left[J^*_{r^t} - J^{\pi^{t}}_{r^t}\right]&\le \frac{\tau}{1-\gamma} \mathbb{E}\left[\E_{s\sim\nu_{r^t}^*} \bs{
            \kl{\pi^{t}(\cdot|s)}{\pi^{t+1}(\cdot|s))}}\right].
    \end{align}
\end{lemma}
\noindent This result extends \cite[Lemma 5]{Cen2022} to the stochastic setting.
\begin{proof}
From Lemma~\ref{lemma:pd_lemma}, it follows that
\begin{align}
     \mathbb{E}\left[J^*_{r^t} - J^{\pi^{t}}_{r^t}\right] &= \dfrac{1}{1-\gamma}\E\bs{\E_{(s,a)\sim\mu_{r^t}^{*}}\bs{A_{r^t}^{\pi^t}(s,a)} - \tau \E_{s\sim\nu_{r^t}^*}\bs{\kl{\pi_{r^t}^*(\cdot|s)}{\pi^{t}(\cdot|s)}}}\\
     &= \dfrac{1}{1-\gamma}\E\left[\E_{(s,a)\sim\mu_{r^t}^{*}}
     \bs{Q_{r^t}^{\pi^t}(s,a) - V_{r^t}^{\pi^t}(s) - \tau \log \pi_{r^t}^*(a|s)}\right]\\
     &= \dfrac{1}{1-\gamma}\E\left[\E_{s\sim\nu_{r^t}^{*}}\bs{\underbrace{\E_{a\sim\pi_{r^t}^{*}(\cdot|s)}
     \bs{Q_{r^t}^{\pi^t}(s,a) + \Delta^t(s,a) - \tau \log \pi_{r^t}^*(a|s)}}_{(A)} - \underbrace{V_{r^t}^{\pi^t}(s)}_{(B)} }\right],\label{eq:sub_opt_gap_policy_AB}
\end{align}
where we used $\Delta^{t}(s,a):=\hat Q^{\pi^t}_{r^t}(s,a)- Q^{\pi^t}_{r^t}(s,a)$ and Equation \eqref{eq:unbiasedQ} in the last step. Next, we bound $(A)$ and $(B)$ separately. For $(A)$ we have by Jensen's inequality
\begin{align}
    &\E_{a\sim\pi_{r^t}^{*}(\cdot|s)}
     \bs{Q_{r^t}^{\pi^t}(s,a) + \Delta^t(s) - \tau \log \pi_{r^t}^*(a|s)} \\
     =& \tau \sum_{a\in\mathcal{A}} \pi^*_{r^t}(a|s)\log \br{
        \dfrac{\exp\br{\br{Q_{r^t}^{\pi^t}(s,a) 
        + \Delta^{t}(s,a)}/\tau}}{\pi^*_{r^t}(a|s)} } \\
    \leq&
    \tau \log \br{
        \sum_{a\in\mathcal{A}}
        \exp\br{\br{Q_{r^t}^{\pi^t}(s,a) 
        + \Delta^{t}(s,a)}/\tau }} = \tau \log Z^t(s). \label{eq:i_kl_bound}
\end{align}
For $(B)$ the definition of the value function and the soft policy iteration update yield
\begin{align}
    V_{r^t}^{\pi^t}(s) &= \E_{a\sim \pi^t(\cdot|s)}\bs{Q_{r^t}^{\pi^t}(s,a) - \tau \log\pi^t(a|s)} \\
    &= \E_{a\sim \pi^t(\cdot|s)}\bs{Q_{r^t}^{\pi^t}(s,a) - \tau \log\pi^{t+1}(a|s)}  - \tau \kl{\pi^{t}(\cdot|s)}{\pi^{t+1}(\cdot|s)}\\
    &= \E_{a\sim \pi^t(\cdot|s)}\bs{\tau\log Z^t(s)-\Delta^{t}(s,a)}  - \tau \kl{\pi^{t}(\cdot|s)}{\pi^{t+1}(\cdot|s)}.
\end{align}
Plugging these bounds for $(A)$ and $(B)$ back into \eqref{eq:sub_opt_gap_policy_AB}, using again that Equation \eqref{eq:unbiasedQ}, we arrive at the desired inequality
\begin{equation}
    \mathbb{E}\left[J^*_{r^t} - J^{\pi^{t}}_{r^t}\right] \leq \dfrac{1}{1-\gamma}\E\left[\E_{s\sim\nu_{r^t}^{*}}\bs{\tau \kl{\pi^{t}(\cdot|s)}{\pi^{t+1}(\cdot|s)}}\right].
\end{equation}
\end{proof}
\begin{lemma}
    \label{lemma:j_lipschitz_reward_fixed_pi} 
    For any reward iterates $r_t,r_{t+1}$ generated by Algorithm \ref{alg:irl} and any policy $\pi\in\Delta_{\Acal}^{\Scal}$, we have
    \begin{align}
        \left|J^{\pi}_{r^t}- J^{\pi}_{r^{t+1}}\right|\leq\frac{2\eta_w\|\phi\|_{\infty}}{(1-\gamma)^2}\label{eq1},\\
        \left|J_{r^t}^*- J_{r^{t+1}}^*\right|\leq\frac{2\eta_w\|\phi\|_{\infty}}{(1-\gamma)^2}.\label{eq2}
    \end{align}  
\end{lemma}
\begin{proof} 
Inequality \eqref{eq1} holds since
    \begin{align}
    \left|J^{\pi}_{r^t}- J^{\pi}_{r^{t+1}}\right|\stackrel{(i)}{=} &\left|
    \E_{\pi}
    \Bigg[
        \sum_{t=0}^{+\infty}
        \gamma^t
        \big(
            r_1(s_t,a_t) - \cancel{\tau \log \pi(a_t|s_t)}
            - r^{t+1}(s_t,a_t) + \cancel{\tau \log \pi(a_t|s_t)}
        \big)
    \Bigg]\right|
    \\ = &\left|
    \E_{\pi}
    \Bigg[
        \sum_{t=0}^{+\infty}
        \gamma^t
        \big(
            r^t(s_t,a_t) - r^{t+1}(s_t,a_t) 
        \big)
    \Bigg]\right|
    \\ \leq &
    \Bigg[
        \sum_{t=0}^{+\infty}
        \gamma^t
        \big\| r^t - r^{t+1} \big\|_\infty
    \Bigg]
    \\ \stackrel{(ii)}{=} &\eta_{w}
    \sum_{t=0}^{+\infty}
    \gamma^t
    \big\|\hat\nabla_w L(\pi,w^t) \big\|_\infty \\
    \stackrel{(iii)}{\le}  &
   \frac{2\eta_w\|\phi\|_{\infty}}{(1-\gamma)^2},
    \end{align}
 where $(i)$ holds by the definition of $J_r^\pi$, $(ii)$ holds by plugging reward updating form and non-expansiveness of projection and $(iii)$ holds by Lemma \ref{Unbiased_A_sampler}.\\
 
\noindent Inequality \eqref{eq2} holds since
\begin{align}
       \left|J_{r^t}^*- J_{r^{t+1}}^*\right|\leq\max_{\pi \in \Pi}\left|J^{\pi}_{r^t}-J^{\pi}_{r^{t+1}}\right|\stackrel{(i)}{\leq}\frac{2\eta_w\|\phi\|_{\infty}}{(1-\gamma)^2},
    \end{align}
    where $(i)$ holds by inequality \eqref{eq1}.
\end{proof} 

\begin{lemma}[Lipschitz continuity of occupancy measure in policy]\label{lem:lipschitz_occ}
    Let $\mu^\pi$ denote the occupancy measure corresponding to the policy $\pi\in\Delta_{\Acal}^{\Scal}$. Then, for any $\pi_1, \pi_2\in\Delta_{\Acal}^{\Scal}$ we have    
    \begin{equation}
        \norm{\mu^{\pi_1} - \mu^{\pi_2}}_1 \leq \dfrac{1}{1-\gamma}\max_s\norm{\pi_1(\cdot|s)-\pi_2(\cdot|s)}_1.
    \end{equation}
    \begin{proof}
    We can upper bound $\norm{\mu^{\pi_1} - \mu^{\pi_2}}_1$ as follows
    \begin{align}
    \norm{\mu^{\pi_1} - \mu^{\pi_2}}_1 &\leq \sum_{s,a} \abs{\nu^{\pi_1}(s)(\pi_1(a|s) - \pi_2(a|s))} + \sum_{s,a} \abs{(\nu^{\pi_1}(s) - \nu^{\pi_2}(s))\pi_2(a|s)}\\
    &\leq \max_s \norm{\pi_1(\cdot|s) - \pi_2(\cdot|s)}_1 + \norm{\nu^{\pi_1} - \nu^{\pi_2}}_1\nonumber,
    \end{align}
    where we used the triangle and Hölder's inequality. From the Bellman flow constraints \cite{Puterman1994}
    \begin{equation}
    \nu^{\pi}(s) = \gamma \sum_{s',a'}P(s|s',a')\mu^{\pi}(s',a') + (1-\gamma) \nu_0(s),
    \end{equation}
    it follows that
    \begin{align}
    \norm{\nu^{\pi_1} - \nu^{\pi_2}}_1 &= \gamma \sum_{s} \abs{\sum_{s',a'}P(s|s',a')(\mu^{\pi_1}(s',a')-\mu^{\pi_2}(s',a'))}\\
    &\leq \gamma \sum_{s',a'}\underbrace{\sum_s P(s|s',a')}_{=1} \abs{\mu^{\pi_1}(s',a')-\mu^{\pi_2}(s',a')}\nonumber\\
    &= \gamma\norm{\mu^{\pi_1}- \mu^{\pi_2}}_1\nonumber,
    \end{align}
    where we again used the triangle inequality. Hence, it follows that
    \begin{equation}
    \max_s \norm{\pi_1(\cdot|s) - \pi_2(\cdot|s)}_1 \geq \norm{\mu^{\pi_1}- \mu^{\pi_2}}_1 - \norm{\nu^{\pi_1}- \nu^{\pi_2}}_1 \geq (1-\gamma) \norm{\mu^{\pi_1}- \mu^{\pi_2}}_1.
    \end{equation}
    \end{proof}
\end{lemma}

\begin{lemma}\label{lem:lip_feat}
    For any two policies $\pi_1,\pi_2\in\Delta_{\Acal}^{\Scal}$, we have
    \begin{equation}
        \|\sigma^{\pi_1} - \sigma^{\pi_2}\|_\infty \nonumber \\
    \leq\frac{2\|\phi\|_\infty}{(1-\gamma)^2} \max_s \|\pi_1(\cdot|s) - \pi_2(\cdot|s) \|_{\text{TV}}.
    \end{equation}
\end{lemma}
\begin{proof}
The proof follows from Hölder's inequality and Lemma~\ref{lem:lipschitz_occ} above:
\begin{align}
    \|\sigma^{\pi_1} - \sigma^{\pi_2}\|_\infty &= \max_{1\leq i \leq k}\max_{s,a}\abs{\dfrac{1}{1-\gamma}\ip{\phi_i}{\mu^{\pi_1}-\mu^{\pi_2}}}\\
    &\leq \dfrac{1}{1-\gamma}\max_{1\leq i \leq k} \norm{\phi_i}_\infty\norm{\mu^{\pi_1}-\mu^{\pi_2}}_1\\
    &\leq \dfrac{\norm{\phi}_\infty }{(1-\gamma)^2}\max_s\norm{\pi_1(\cdot|s)-\pi_2(\cdot|s)}_1\\
    &\leq \dfrac{2\norm{\phi}_\infty }{(1-\gamma)^2}\max_s\norm{\pi_1(\cdot|s)-\pi_2(\cdot|s)}_{TV}.
\end{align}
\end{proof}

\end{document}